\documentclass[twoside]{article}

\usepackage{amssymb}
\usepackage[utf8]{inputenc}
\usepackage{newunicodechar}
\newunicodechar{̧}{\c{}}
\usepackage{algorithm}
\usepackage{color}
\usepackage{algorithmic}
\usepackage{hyperref}

\usepackage{bm}
\usepackage{graphicx}
\usepackage{booktabs,multirow,makecell}
\usepackage{amsthm}
\newtheorem{theorem}{Theorem}[section]

\newtheorem{lemma}[theorem]{Lemma}
\newtheorem{corollary}[theorem]{Corollary}
\newtheorem{definition}[theorem]{Definition}
\newtheorem{assumption}[theorem]{Assumption}

\usepackage{multirow}
\usepackage{xcolor,soul}
\usepackage{subfigure}

\newcommand{\R}{\mathbb{R}}
\newcommand{\one}{\textbf{1}}
\newcommand{\Norm}{\mathrm{Norm}}

\newcommand{\meanpm}[2]{$#1~{\pm~ #2}$}       
\newcommand{\bestpm}[2]{$\mathbf{#1~{\pm~ #2}}$} 
\usepackage[accepted]{aistats2026}
\begin{document}

%

%

\twocolumn[

\aistatstitle{Robust Generalization with Adaptive Optimal Transport Priors for Decision-Focused Learning}

\aistatsauthor{ Haixiang Sun \And Andrew L. Liu }

\aistatsaddress{Purdue University} ]

\begin{abstract}
Few-shot learning requires models to generalize under limited supervision while remaining robust to distribution shifts. Existing Sinkhorn Distributionally Robust Optimization (DRO) methods provide theoretical guarantees but rely on a fixed reference distribution, which limits their adaptability. We propose a Prototype-Guided Distributionally Robust Optimization (PG-DRO) framework that learns class-adaptive priors from abundant base data via hierarchical optimal transport and embeds them into the Sinkhorn DRO formulation. This design enables few-shot information to be organically integrated into producing class-specific robust decisions that are both theoretically grounded and efficient, and further aligns the uncertainty set with transferable structural knowledge. Experiments show that PG-DRO achieves stronger robust generalization in few-shot scenarios, outperforming both standard learners and DRO baselines.
\end{abstract}

\section{Introduction}

Few-shot learning aims to enable models to generalize to new tasks with only a handful of labeled examples. In many real-world scenarios, one may have access to abundant regular data from common conditions, but only a very limited number of examples under rare or extreme circumstances. To prepare models for such situations, it is crucial to design approaches that can learn effectively even when supervision is scarce. While metric-based \cite{NIPS2017_cb8da676, Vinyals2016MatchingNF, Wang2018CosFaceLM} and meta-learning approaches \cite{Chen2020MetaBaselineES, Finn2017ModelAgnosticMF} have shown strong results in clean benchmarks, they struggle in practice when data are imbalanced, shifted, or adversarially perturbed. This fragility directly impacts the reliability of the decision rules derived from few-shot learners.

In practice, the data distributions encountered by few-shot learners are rarely stable. Even small perturbations in input space, regardless of sensor noise, domain shifts, or adversarial manipulation, can push examples across tight decision margins and yield unstable predictions \cite{hendrycks2018benchmarking, Koh2020WILDSAB, madry2018towards}. Robust decision making therefore requires models not only to interpolate among a few examples, but also to anticipate the worst-case deviations that could occur around them \cite{Esfahani2015DatadrivenDR, Goldblum2019RobustFL}. Preparing for such perturbations reduces overfitting to idiosyncrasies of the limited support set and encourages invariances that transfer across environments. In this sense, robustness is not a luxury but a necessity: it determines whether few-shot learning methods can survive the realities of open-world deployment.


A principled approach to ensure robustness is distributionally robust optimization, which minimizes the worst-case risk over an ambiguity set centered at the empirical support distribution \cite{Kuhn2024DistributionallyRO,Rahimian2019FrameworksAR}. The ambiguity set can be defined through different notions of distance between the historical data and a surrogate distribution. A widely used choice is the Wasserstein distance, and its entropic regularization leads to the Sinkhorn distance, which offers favorable computational properties \cite{Blanchet2019ConfidenceRI,Cuturi2013SinkhornDL,Esfahani2015DatadrivenDR}. Existing studies, such as \cite{azizian2023regularization,wang2025sinkhorn}, reformulate the resulting dual problems into low-dimensional parameterizations, making optimization tractable. Nevertheless, these approaches remain tied to a single predefined reference distribution, which restricts their ability to adapt when the data distribution shifts. This fixed-reference design represents a fundamental limitation in real-world few-shot learning scenarios, where the support distribution is sparse and rarely representative of future test conditions.


In this work, we address two central challenges of robust decision making under limited data: achieving reliable generalization from scarce supervision and maintaining stability under distributional shift. We propose Prototype-Guided Distributionally Robust Optimization (PG-DRO), a class-adaptive framework that constructs class-specific priors via optimal transport from base prototypes and incorporates them directly into the robust decision rule. By aligning novel classes with structural knowledge distilled from a rich base dataset, PG-DRO improves generalization when only a handful of labeled examples are available. At the same time, it explicitly accounts for worst-case perturbations within each class, enhancing robustness and generating informative outliers that highlight the decision boundaries most vulnerable to shift. These properties make PG-DRO valuable for domains where models must extrapolate from scarce and noisy observations while remaining stable under unexpected change across many real-world applications.

\paragraph{Contribution}
\begin{itemize}

    \item We introduce Prototype-Guided Distributionally Robust Optimization (PG-DRO), a novel decision-making framework for low-data regimes that leverages hierarchical optimal transport to construct class-adaptive priors within entropic DRO, enabling scalable and robust learning under limited supervision. 

    \item We demonstrate that this framework yields strong robust generalization: it maintains accuracy under perturbations and distribution shifts, particularly benefiting few-shot minority classes.

    \item We also provide theoretical analysis establishing conditions for robust optimization with adaptive priors, and validate our framework through extensive experiments showing consistent gains in both accuracy and robustness under distributional shifts.
\end{itemize}
\section{Related works}

\paragraph{Few-shot Learning}
Few-shot learning (FSL) aims to rapidly adapt a model to novel classes from only a handful of labeled examples by leveraging transferable structure learned from base tasks \cite{Finn2017ModelAgnosticMF,Snell2017PrototypicalNF}. Beyond the canonical setting, a growing body of work investigates domain-adaptive FSL, where base and novel classes come from different or even mismatched domains, requiring methods to bridge distribution gaps and improve cross-domain transfer \cite{10377083,Pal2023DomainAF,Zhao2020DomainAdaptiveFL}. More recently, researchers have also begun to study robust FSL, seeking resilience to adversarial perturbations and distributional shifts \cite{goldblum2020adversarially,sagawadistributionally,wang2021on}. Such robustness-oriented approaches have been explored across downstream tasks, including natural language understanding \cite{nookala2023adversarial}, image classification \cite{9878634}, and reinforcement learning \cite{greenberg2023train}, yet they often remain limited by task-agnostic assumptions about uncertainty sets, leaving open the question of how to construct class-adaptive robust priors under few-shot regimes.

\paragraph{Optimal Transport}

Optimal Transport (OT) provides a geometric framework for comparing probability measures by modeling the minimal cost of transporting mass between them \cite{Cuturi2013SinkhornDL,Peyr2018ComputationalOT}. A widely used extension is the entropically regularized formulation, which introduces a smoothing term and leads to the so-called Sinkhorn distance \cite{Carlier2015ConvergenceOE,pmlr-v32-cuturi14}. Formally, the entropic OT problem seeks a coupling $T\in\mathbb{R}^{B\times N}$ solving
\begin{equation}
\label{eq_ot}
\min_{T\geq 0};\langle C,T\rangle+\varepsilon H(T),\quad \text{s.t. }T\mathbf{1}=\alpha,~~T^\top\mathbf{1}=\beta,
\end{equation}
where $H(T)=-\sum_{bn}T_{bn}\log T_{bn}$ is the entropy and $\alpha,\beta$ are the prescribed marginals.

OT has since become a versatile tool across several domains. In domain adaptation, OT-based methods align source and target distributions either at the feature level or jointly with labels \cite{chang2022unified,Courty2014OptimalTF,Damodaran2018DeepJDOTDJ}, with extensions to multi-source, target-shift settings \cite{Redko2018OptimalTF} and deep architectures \cite{ijcai2020p299,9157821}. OT also works in representation learning, which underpins cross-domain embedding alignment \cite{Chen2020GraphOT}, self-supervised objectives \cite{caron2020unsupervised,pmlr-v202-shi23j}, and structure-preserving mappings \cite{AlvarezMelis2018GromovWassersteinAO}. Beyond these, OT has been applied to downstream tasks ranging from generative modeling with Wasserstein losses \cite{Deshpande2018GenerativeMU,Kolouri2019GeneralizedSW,Simsekli2018SlicedWassersteinFN} to structured prediction and sequence modeling. Beyond classical OT applications, several works adapt OT for few-shot learning. Hierarchical OT adapts base-class statistics to novel classes for better priors \cite{guo2022adaptive}, while \cite{guo2022learning,NEURIPS2019_8b5040a8,9156564} use OT for structured representation and prototype aggregation. Our work resonates with these methods in using hierarchical OT to capture cross-class relations, but differs by embedding it into a Sinkhorn DRO framework to obtain class-adaptive robust logits rather than solely performing distribution calibration or metric alignment.

\paragraph{Robust Optimization} It is formulated as a min-max
problem mathematically, 
\begin{equation}
    \min_{\theta\in\Theta}\ \sup_{P\in \mathcal{U}(\mu,\nu)}
    \ \mathbb{E}_{(\mathbf{x},\mathbf{y})\sim P}\big[\ell_\theta(\mathbf{x},\mathbf{y})\big],
\end{equation}
where the ambiguity set $\mathcal{U}$ collects all candidate joint distributions $P$ whose distance to a reference distribution does not exceed a prescribed radius.

For current DRO, common choices for the ambiguity set $\Theta$ include $\phi$-divergences \cite{10.5555/3157096.3157344,Namkoong2016VariancebasedRW}, Wasserstein balls \cite{Esfahani2015DatadrivenDR,Gao2016DistributionallyRS}, and kernel metrics such as MMD \cite{Feydy2018InterpolatingBO,staib2019distributionally,zhu2021kernel}. Wasserstein DRO leverages the geometry of the sample space and provides strong guarantees, but its dual entails a hard supremum, with worst-case distributions degenerating to finite discrete measures, problematic under small support \cite{Gao2016DistributionallyRS}. Entropic regularization smooths the Wasserstein dual’s supremum into a log-sum-exp, yielding the Sinkhorn distance solvable via fast matrix scaling \cite{Cuturi2013SinkhornDL,Peyr2018ComputationalOT}. Recent analyses further show that this regularization leads to a single scalar dual variable while preserving convexity and tractability \cite{azizian2023regularization,wang2025sinkhorn}. Decision-focused DRO trains prediction and decisions jointly \cite{chenreddy2024endtoend,costa2023distributionally, Donti2017TaskbasedEM,Ma2024DifferentiableDR}, while adversarial training links to Wasserstein-DRO, unifying attacks with distributional models \cite{bai2023wasserstein,Deng2020AdversarialDT,Schmidt2018AdversariallyRG} and providing certificates \cite{sinha2018certifiable}.

However, despite their differences, all these approaches posit an ambiguity set centered around a single, fixed reference distribution, which in few-shot regimes can only be estimated from a handful of supports and thus is prone to misalignment with rare or shifted classes. Our method instead adapts the reference per class via adaptive OT, thereby retaining Sinkhorn DRO’s tractability while aligning the uncertainty set with transferable base-class structure.

\section{Preliminary}

We begin by recalling the formulation of distributionally robust optimization under entropic optimal transport. Let $\mathcal{X}$ denote the input space, let $\mu\in\mathcal{P}(\mathcal{X})$ be the empirical distribution of observed inputs, and let $\ell:\mathcal{Y}\rightarrow\mathbb{R}$ be the prediction loss function defined on the output space $\mathcal{Y}$. Consider a reference measure $\nu$ on $\mathcal{Y}$ and a ground cost function $c:\mathcal{X}\times\mathcal{Y}\rightarrow\mathbb{R}_{+}$ that encodes admissible perturbations between input $\mathbf{x}\in\mathcal{X}$ and output $\mathbf{y}\in\mathcal{Y}$. The entropic optimal transport distance between two distributions $\mu$ and $\nu$ is defined as
\begin{equation}
\begin{aligned}
W_\varepsilon(\mu,\nu)=\inf_{\gamma\in\Gamma(\mu,\nu)}&\mathbb{E}_{(\mathbf{x},\mathbf{y})\sim\gamma}\big[c(\mathbf{x},\mathbf{y})\big]\\&+\varepsilon\mathrm{KL}\left(\gamma\Vert \mu\otimes \nu\right),
\end{aligned}
\end{equation}
where $\Gamma(\mu,\nu)$ denotes the set of couplings between $\mu$ and $\nu$, $\varepsilon>0$ is the entropic regularization parameter, and $\mathrm{KL}(\cdot\Vert\cdot)$ is the Kullback–Leibler divergence.

\begin{definition}[Sinkhorn DRO]
The worst-case expected function $f$ over the Sinkhorn ball $\{\mu\in\mathcal{P}(\mathcal{X}):W_\varepsilon(\mu,\nu)\leq\rho\}$ admits the form
\begin{equation}
\label{dual_sinkhorn}
  V=\sup\limits_{W_\varepsilon(\mu,\nu)\leq\rho}\ \mathbb{E}_{\mathbf{x}\sim P}\mathbb{E}_{\mathbf{y}\sim\gamma_{\mathbf{x}}}\big[f(\mathbf{y})\big],
\end{equation}
where $\rho>0$ specifies the radius of the uncertainty set, $\nu$ is a fixed reference measure on $\mathcal{Y}$, and $\gamma_{\mathbf{x}}$ is the conditional distribution of $\mathbf{y}$ given $\mathbf{x}$ under some coupling $\Pi\in\Gamma(\mu,\nu)$.
\end{definition}

By strong duality, this problem can be equivalently expressed in terms of a convex dual formulation
\begin{equation}
V_D(\lambda)=\lambda\rho+\lambda\varepsilon\mathbb{E}_{\mathbf{x}}\left[\log\mathbb{E}_{\mathbf{y}\sim Q^{\nu}_{\varepsilon,\mathbf{x}}}\exp\left(\frac{f(\mathbf{y})}{\lambda\varepsilon}\right)\right],
\end{equation}
where $\lambda>0$ is a scalar dual variable, and $Q^{\nu}_{\varepsilon,\mathbf{x}}$ is a Gibbs kernel defined as
\begin{equation}
  \mathrm{d}Q^{\nu}_{\varepsilon,\mathbf{x}}(\mathbf{y})
  =\frac{\exp\left(-c(\mathbf{x},\mathbf{y})/\varepsilon\right)}{Z(\mathbf{x})}\mathrm{d}\nu(\mathbf{y}),
\end{equation}
with the normalizing constant $Z(\mathbf{x})=\int \exp\left(-c(\mathbf{x},\mathbf{y})/\varepsilon\right)\mathrm{d}\nu(\mathbf{y})$. 
This dual formulation highlights that Sinkhorn DRO reduces the infinite-dimensional adversarial optimization to the minimization of a one-dimensional convex function in $\lambda$, where robustness is achieved through the log-moment generating function of $\ell$ under the kernelized distribution $Q^{\nu}_{\varepsilon,\mathbf{x}}$. Intuitively, the adversary inflates the expected loss by softmax-smoothing over plausible perturbations of each sample, controlled jointly by the ground cost $c$ and the reference distribution $\nu$.

\begin{lemma}[Convexity and uniqueness of the dual minimizer; cf.~\cite{wang2025sinkhorn}, Thm.~1(III) and Lemma~3]
\label{lem:convexity-unique}
$V_D(\lambda)$ is convex in $\lambda$ on $[0,\infty)$. Moreover, unless $\ell$ is $\nu$-a.s.\ constant, $V_D(\lambda)$ is \emph{strictly} convex on $(0,\infty)$ and thus admits a unique minimizer $\lambda^\star>0$. In the degenerate case where $\ell$ is $\nu$-a.s.\ constant, the unique minimizer is attained at $\lambda^\star=0$.
\end{lemma}

Since $V_D(\lambda)$ is strictly convex and satisfies $V_D(\lambda)\ge \lambda\rho$, it is coercive for $\rho>0$ and hence $V_D(\lambda)\to\infty$ as $\lambda\to\infty$. Therefore the minimizer not only exists but is also unique, with $\lambda^\star=0$ exactly in the degenerate constant-loss case. This ensures a well-posed one-dimensional optimization for any fixed prior $\nu$. However, the statistics still hinge on how well $\nu$ reflects the target task. In few-shot settings, generally, such a rigid prior cannot encode any task-specific structure. To remedy this, we introduce hierarchical OT priors ${\nu_c}$ for different classes that transfer geometry and distributional knowledge from base classes, thereby enabling a class-specific robust decision guidance.

\section{Methodology}

\paragraph{Adaptive Optimal Transport Priors.}

While Sinkhorn DRO offers a principled route to distributional robustness, a fixed reference distribution $\nu$ is agnostic to the downstream few-shot task. It ignores class-specific geometry, which leads to a mismatch of ambiguity set under scarce supervision. To make robustness class-aware by leveraging the few-shot support information, we transfer knowledge from abundant base classes and construct class-adaptive priors via a hierarchical optimal transport procedure \cite{guo2022adaptive}.

Let the source domain contain $B$ classes, indexed by $b=1,\ldots,B$, with few-shot prototypes $\{\mathbf{x}_{b,i}\}_{i=1}^{m_b}$ for class $b$, class mean $\mu_b=\tfrac{1}{m_b}\sum_{i=1}^{m_b}\mathbf{x}_{b,i}$, and covariance $\Sigma_b=\tfrac{1}{m_b-1}\sum_{i=1}^{m_b}(\mathbf{x}_{b,i}-\mu_b)(\mathbf{x}_{b,i}-\mu_b)^\top$.
For a new few-shot task with labeled support $S_k=\{(\tilde{\mathbf{x}}_n,\mathbf{y}_n)\}_{n=1}^{N}$, we compute for each support example a soft-min Sinkhorn cost to every base class
\begin{equation}
  C_b(\tilde{\mathbf{x}}_n)= -\varepsilon_{\mathrm{s}} \log \sum_{i=1}^{m_b}\exp\!\Big(-\frac{1}{\varepsilon_{\mathrm{s}}}\, c(\tilde{\mathbf{x}}_n,\mathbf{x}_{bi})\Big),
\end{equation}
where $c(\cdot,\cdot)$ is the ground cost and $\varepsilon_{\mathrm{s}}$ is the entropic temperature. Stacking $\{C_b(\tilde{\mathbf{x}}_n)\}$ gives a cost matrix $\mathbf{C}\in\mathbb{R}^{B\times N}$.

At the class level, we solve the entropic OT problem \eqref{eq_ot} on $\mathbf{C}$ to obtain a transport plan $\mathbf{T}\in\mathbb{R}^{B\times N}$ that aligns support examples with base classes. Aggregating mass over the support of class $c$ yields mixture weights
\begin{equation}
  \tilde{w}_{bc}= \frac{\sum_{n:\,\ell_n=c} T_{bn}}{\sum_{b'=1}^{B}\sum_{n:\,\ell_n=c} T_{b'n}},
\end{equation}
and the resulting class-adaptive prior is the Gaussian mixture
\begin{equation}
  \nu_c \;=\; \sum_{b=1}^{B} \tilde{w}_{bc}\,\mathcal{N}(\mu_b,\Sigma_b).
\end{equation}

In summary, the hierarchical OT procedure yields class-adaptive prototype-guided priors $\{\nu_c\}$. Each prior is explicitly anchored on transferable prototypes from base classes and refined by the few-shot support set. Rather than serving as a generic preprocessing step, these priors act as structured guides that can be injected into the Sinkhorn DRO dual.

\paragraph{Prototype-Guided DRO (PG-DRO)}
Plugging $\nu_c$ into the Sinkhorn DRO kernel $Q^{\nu_c}_{\varepsilon,\mathbf{x}}$ yields the class-specific robust logit
\begin{equation}
\label{eq:class_dual}
V_c(\mathbf{x};\theta) =\min_{\lambda\ge 0}\Big\{\lambda\rho+\lambda\varepsilon\log\mathbb{E}_{\mathbf{y}\sim Q^{\nu_c}_{\varepsilon,\mathbf{x}}}\exp\big(\tfrac{f_c(\mathbf{y};\theta)}{\lambda\varepsilon}\big)\Big\}.
\end{equation}
This yields class-specific robust logits $V_c(\mathbf{x};\theta)$, which capture adversarial uncertainty in a manner consistent with the structural prior of class $c$, where $\theta$ denotes the parameter. Since the objective in $\lambda$ is smooth and strictly convex, \eqref{eq:class_dual} can be solved efficiently via a one-dimensional Newton step per class. Compared with a single fixed prior $\nu$, prototype-guided priors sharpen the ambiguity set of $\nu_c$ by some prototypes around class-consistent regions, align worst-case perturbations $\mathbf y \sim Q^{\nu_c}_{\varepsilon,\mathbf x}$ with transferable geometry, and deliver decision-focused gains through the robust logit $V_c(\mathbf x;\theta)$. This turns robustness from a global assumption into a class-level inductive bias.

\begin{corollary}[Multi-class convexity]\label{cor:multiclass-convexity}
For each class $c$, the dual objective $V_c(\mathbf{x};\theta)$ is (strictly) convex in $\lambda_c$ on $(0,\infty)$ in the sense of Lemma \ref{lem:convexity-unique}. Consequently, the joint objective
\begin{equation}
\begin{aligned}
&\mathcal{V}(\lambda_1,\ldots,\lambda_C)\\
&:=\sum_{c=1}^C \Big\{\lambda_c\rho+\lambda_c\varepsilon\log
\mathbb{E}_{\mathbf{y}\sim Q^{\nu_c}_{\varepsilon,\mathbf{x}}}\exp\big(f_c(\mathbf{y};\theta)/(\lambda_c\varepsilon)\big)\Big\}
\end{aligned}
\end{equation}
is convex on $(0,\infty)^C$. Moreover, unless $f_c(\cdot;\theta)$ is $\nu_c$-a.s.\ constant for every $c$, $\mathcal{V}$ is strictly convex in each non-degenerate coordinate and the corresponding minimizer components $\lambda_c^\star$ are unique; if $f_c(\cdot;\theta)$ is a.s.\ constant for some $c$, then the unique minimizer in that coordinate is $\lambda_c^\star=0$.
\end{corollary}

\paragraph{Decision rule and training objective.}
With class-specific robust logits in hand, our prediction is $\hat{c}=\arg\max_c V_c(\mathbf{x};\theta)$. Training minimizes the softmax cross-entropy applied to these robust logits:
\begin{equation}
\label{ce_loss}
\mathcal{L}_{\mathrm{CE}}(\mathbf{x},c^\star;\theta)
= -\log \frac{\exp\!\big(V_{c^\star}(\mathbf{x};\theta)\big)}{\sum_{c=1}^{C}\exp\!\big(V_{c}(\mathbf{x};\theta)\big)}.
\end{equation}
Note that this coincides with the standard cross-entropy when $V_c$ are the raw logits from the classifier head. Our goal is to make the decision reliable under shift. Since each $V_c(\mathbf{x};\theta)$ in \eqref{eq:class_dual} is the worst-case class score within a Sinkhorn ball shaped by the class-adaptive prior $\nu_c$, the likelihood $V_c(\mathbf{x};\theta)$ is guarded against per-class perturbations. Leveraging few-shot supports to form $\nu_c$ aligns the ambiguity set with transferable base-class structure, keeping novel-class predictions reliable under domain noise and scarcity. Besides, PG-DRO increases it increases the robust margin $V_{c^\star}(\mathbf{x};\theta)-\max_{c\ne c^\star}V_c(\mathbf{x};\theta)$ and thereby lowers worst-case misclassification.

To study how adaptive OT priors affect the robust objective, we track the Sinkhorn DRO value $V_\nu(\theta)$ as the class prior $\nu$ is updated. Here  for brevity, we fix a target class $c$ and omit the subscript $c$. Denote $w^{(t)}$ as the mixture weights over base prototypes in iteration $t$ and $\nu^{(t)}=\sum_b w^{(t)}_{b}\mu_b$, the update is given by a relaxed fixed-point iteration with the stochastic OT map $\widehat T$. Based on some wild and common assumptions, we could obtain the following contraction property:
\begin{theorem}[Contraction of $V$ under adaptive OT]\label{thm:contraction}
Let $w^{(t+1)}=(1-\eta_t)w^{(t)}+\eta_t\widehat T\big(w^{(t)}\big)$ with $\eta_t\in(0,1)$,
and define $\nu^{(t)}=\sum_b w^{(t)}_{b}\mu_b$ and $\Delta_t(\theta):=|V_{\nu^{(t)}}(\theta)-V_{\nu^\star}(\theta)|$.
Suppose the OT map $T$ is locally contractive near $w^\star$ with contraction rate $\kappa\in(0,1)$. Then, for all $t$ in the contraction neighborhood,
\begin{equation}
    \Delta_{t+1}(\theta)\ \le\ (1-\eta_t\kappa)\Delta_t(\theta)\ +\ \mathcal O\!\big(N^{-1/2}\big),
\end{equation}
where $\kappa$ is the contraction constant of the population map $T$ around $w^\star$, and $N$ is the number of few-shot samples.
\end{theorem}

The proof is provided in Appendix \ref{supp:exp}. Theorem~\ref{thm:contraction} shows that the adaptive prior update drives the robust value $V_{\nu^{(t)}}(\theta)$ toward the oracle $V_{\nu^\star}(\theta)$ at a linear rate $1-\eta_t\kappa$, modulo a stochastic term $\mathcal O(N^{-1/2})$ arising from few-shot estimation. So whenever the number of prototypes $N\to\infty$ and $\sum_t \eta_t=\infty$, the gap $\Delta_t(\theta)$ vanishes. In particular, the adaptive prior family concentrates around $\nu^\star$ faster than any fixed-reference alternative, which under the Lipschitz continuity of $V_\nu(\theta)$ in $W_1(\nu,\cdot)$ directly supports the worst-case excess-risk comparison.

\begin{theorem}[Consistency]
Assume the ground cost $c(x,y)$ is continuous and bounded below, $f_\theta(\mathbf{x},\mathbf{y})$ is continuous in $\mathbf{x}$, $\mathbf{y}$ and locally Lipschitz in $\theta$, the output space $\mathcal Y$ is compact, then for each $\mathbf{x}$ and class $c$,
\begin{equation}
V_c^{(N)}(\mathbf{x};\theta)\to V_c^\star(\mathbf{x};\theta)~~\text{and}~~
\lambda_c^{(N)}(\mathbf{x})\to\lambda_c^\star(\mathbf{x}),
\end{equation}
where $V_c^{(N)}$ (resp.\ $V_c^\star$) is the Sinkhorn DRO dual with kernel induced by $\nu_c^{(N)}$ (resp.\ $\nu_c^\star$), and $\lambda_c^{(N)}$ (resp.\ $\lambda_c^\star$) is any minimizer of the corresponding dual variable.
\end{theorem}
This establishes robust risk consistency: solutions trained with finitely many supports converge to those of the population-level prototype-guided objective, thereby guarantee the validity of robust decisions and heir potential for generalization under distributional shift.

In summary, by equipping each class with an adaptive prior $\nu_c$ via adaptive OT, PG-DRO embeds class-specific reference distributions into Sinkhorn DRO. This unifies hierarchical priors with distributionally robust optimization, yielding robustness that adapts to the variability of novel classes, improves generalization under limited supervision, and strengthens the reliability of decision making in few-shot scenarios.

\section{Experiments}

We aim to answer the following Research Questions (RQs) through our experiments:
\begin{itemize}
    \item \textbf{RQ1:} Does PG-DRO improve generalization under distribution shifts compared to standard methods?
    \item \textbf{RQ2:} Can PG-DRO provide stronger robustness and decision quality in worst-case scenarios, particularly for weakly-supervised and severely shifted classes?
    \item \textbf{RQ3:} Does the adaptive OT prior in PG-DRO effectively generalize the cross-class knowledge to enhance the robustness of few-shot learners under limited supervision?
\end{itemize}

We next investigate these questions on both synthetic datasets and real benchmark datasets.

\subsection{Simulation}

To answer \textbf{RQ1}, we construct synthetic datasets that simulate distributional shifts between a source domain and a target domain. In the source domain, each of the $C$ latent classes is modeled as a Gaussian
\begin{equation*}
\mathbf{x} \sim \mathcal{N}(\mu_c, \Sigma_c),  \quad c \in \{1,\dots,C\},
\end{equation*}
yielding supervised pairs $(\mathbf{x},\mathbf{z})$.

The target domain contains the same $C$ classes, but their distributions are perturbed versions of the source ones. 
Specifically, we generate target Gaussians by applying
\begin{equation*}
\mu_c^\prime = \mu_c + \lambda_{\text{mean}}\Delta\mu_c, 
\qquad \Sigma_c^\prime = \mathbf{R}\big(\lambda_{\text{cov}}\Sigma_c\big)\mathbf{R}^\top,
\end{equation*}
where $\Delta\mu_c$ denotes a mean displacement vector, $\mathbf{R}$ a random rotation matrix, and both are controlled by the hyperparameters $\lambda_{\text{mean}}$ and $\lambda_{\text{cov}}$. 
To emulate the few-shot adaptation scenario, we retain only a limited number of labeled support samples per class in the target domain.

\begin{figure}[t]
  \centering
  \subfigure[Source domain per class $\mathcal{N}(\mu_c,\Sigma_c)$ \label{fig:src}]{
    \includegraphics[width=.475\linewidth]{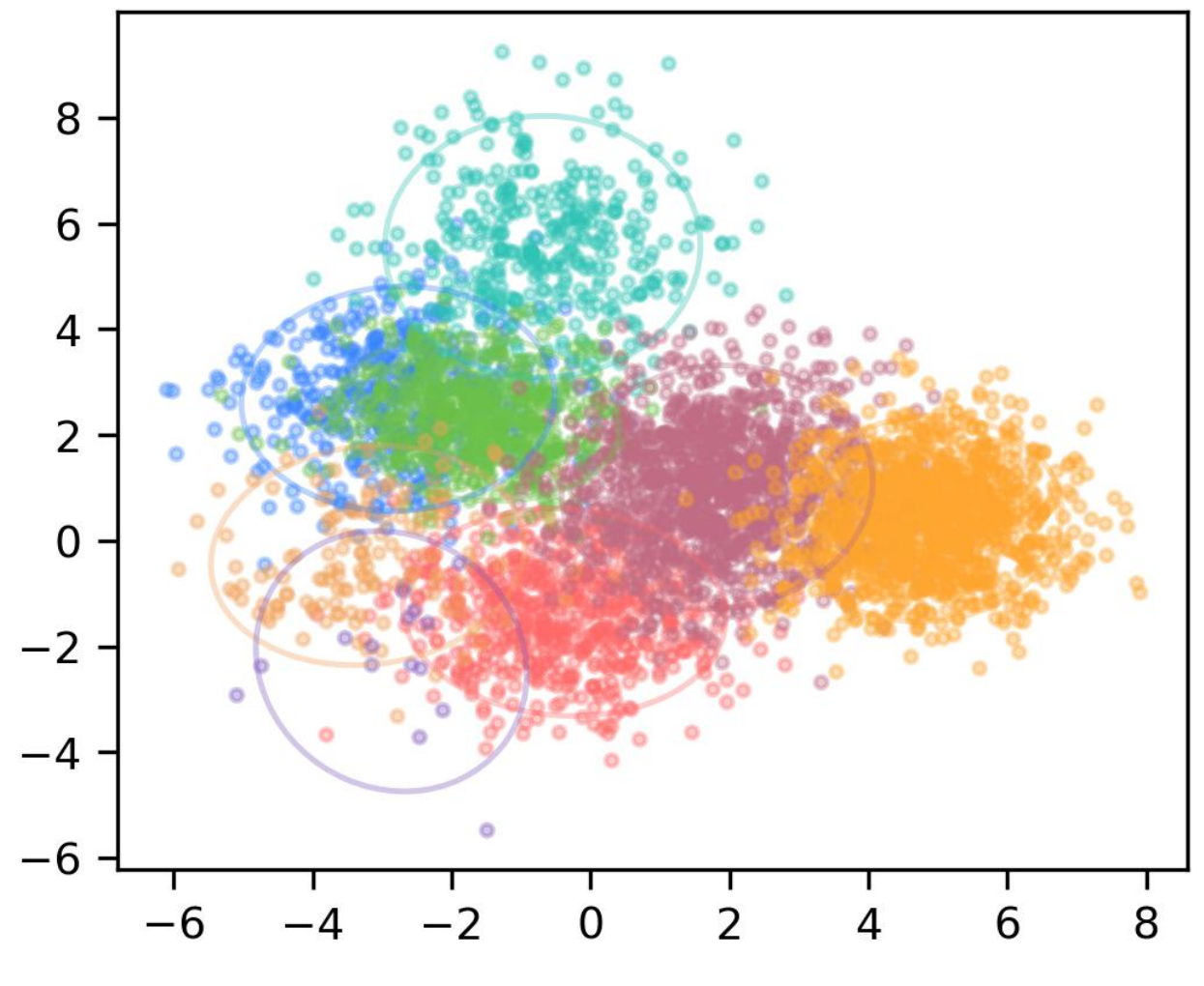}}
  \hfill
  \subfigure[Target domain with $K$-shot supports\label{fig:tgt}]{
    \includegraphics[width=.475\linewidth]{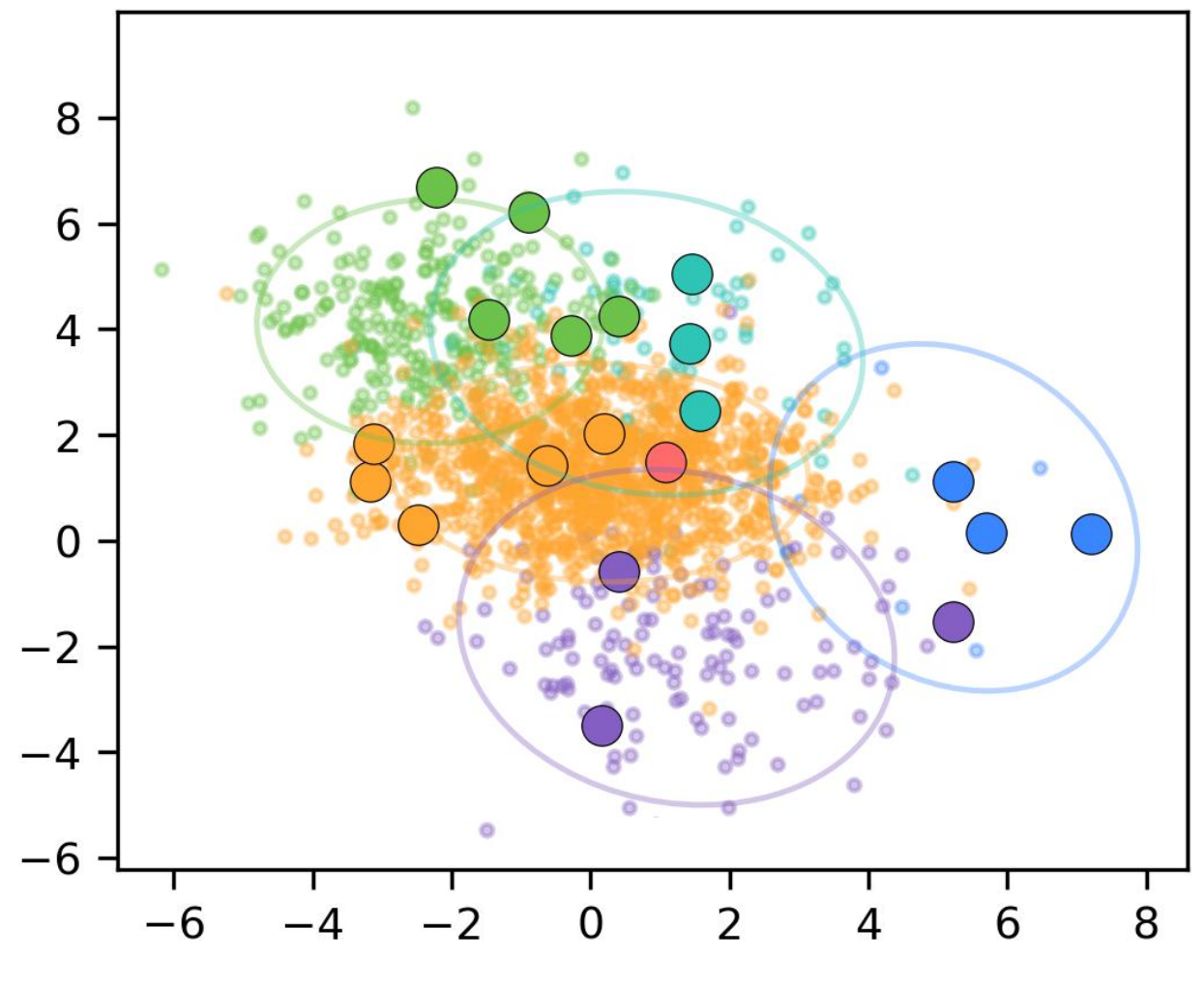}}
  \caption{Comparison of source and target domain with 2D projection}
  \label{fig:geom-paired}
\end{figure}

\begin{figure*}
    \centering
    \includegraphics[width=1\linewidth]{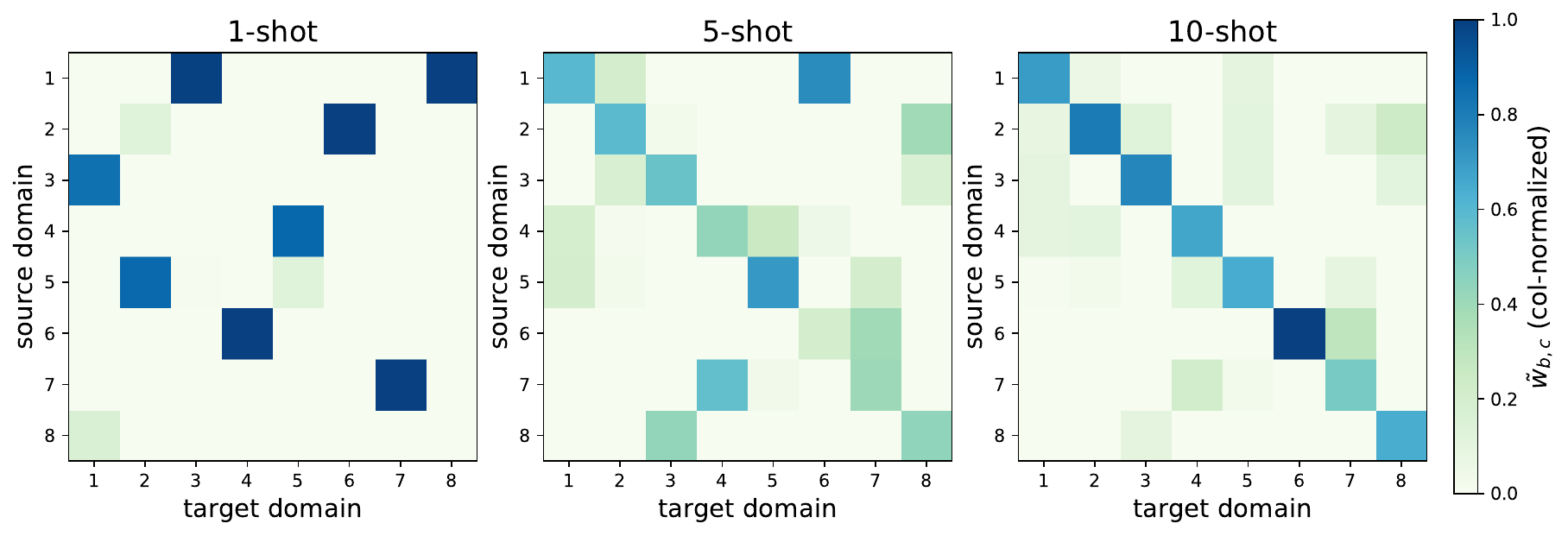}
    \caption{Normalized weights $\tilde{w}_{bc}$ for adaptive OT from source to target domain}
    \label{fig_heatmap}
\end{figure*}
Support $P_{\mathrm{s}}$ and $P_{\mathrm{t}}$ denote source and target domain respectively, we compare PG-DRO against two baselines.
\begin{itemize}
\item \textbf{Empirical Risk Minimization (ERM)}.
    ERM trains predictors only on labeled source-domain samples:
    \begin{equation}
    \label{loss_erm}
    \mathcal{L}_{\mathrm{ERM}} 
    = \mathbb{E}_{(\mathbf{x}_{\mathrm{s}},\mathbf{z}_{\mathrm{s}})\sim P_{\mathrm{s}}}
      \big[\ell\big(\mathbf{z}_{\mathrm{s}}, f_\theta(\mathbf{x}_{\mathrm{s}})\big)\big],
    \end{equation}
    where $\ell$ evaluates decision quality, instantiated as cross-entropy for classification and squared error for regression.

\item \textbf{Classical OT adaptation}.
    OT learns a transport map $\mathcal{T}$ that maps source distributions to target ones, minimizing
    \begin{equation}
    \label{loss_ot}
    \mathcal{L}_{\mathrm{OT}} 
    = \mathbb{E}_{(\mathbf{x}_{\mathrm{s}},\mathbf{z}_{\mathrm{s}})\sim P_{\mathrm{s}}, (\mathbf{x}_{\mathrm{t}},\mathbf{z}_{\mathrm{t}})\sim P_{\mathrm{t}}}
       \big[\ell\big(\mathbf{z}_{\mathrm{t}}, f_\theta(\mathcal{T}(\mathbf{x}_{\mathrm{s}}))\big)\big],
    \end{equation}
    where $(\mathbf{x}_{\mathrm{s}},\mathbf{z}_{\mathrm{s}})$ with $\mathbf{z}_{\mathrm{s}}=\mathbf{z}_{\mathrm{t}}=c$ and $(\mathbf{x}_{\mathrm{t}},\mathbf{z}_{\mathrm{t}})$ are paired from the same class in source and target.
\end{itemize}

To answer $\bf RQ2$, we assess model performance using two complementary metrics. Average accuracy over all target-domain test samples reflects overall generalization under distribution shift. We also report the worst-$10\%$ accuracy, defined as the accuracy restricted to the lowest-performing $10\%$ of classes. This criterion highlights whether a method can sustain reliable predictions even for severely shifted or minority classes. For $\bf RQ3$, we draw a heatmap to show the concentration degree to see if the adaptive OT could capture the mapping through the learned transport weights.

\paragraph{Few-shot simulation}
To emulate realistic prior shift, we sample domain-specific class proportions from Dirichlet distributions. The source domain uses a nearly uniform prior drawn from $\mathrm{Dir}(1)$, while the target domain adopts a long-tailed prior drawn from $\mathrm{Dir}(\alpha_{\mathrm{test}})$ with $\alpha_{\mathrm{test}}\ll1$. Because such priors naturally yield highly imbalanced class frequencies, the resulting support sets contain very few prototypes for minority classes, producing a challenging head–tail few-shot scenario.

Based on this framework, we evaluate PG-DRO against other baselines on both classification and regression tasks to assess its decision-making ability. It is worth noting that all evaluations are performed on the target domain test set, ensuring that our reported metrics directly reflect generalization and robustness under domain shift. And all results are presented as mean $\pm$ standard deviation over repeated runs. All experiments are conducted on a machine equipped with 13th Gen Intel(R) Core(TM) i9-13900HX CPU (24 cores) and NVIDIA A100 GPU (40GB). The code is available at

\begin{table}[t]
\centering
\setlength{\tabcolsep}{4pt}
\renewcommand{\arraystretch}{1.05}
\caption{Results under different levels of disturbance}\vspace{2pt}
\resizebox{\columnwidth}{!}{%
\begin{tabular}{c l c c}
\toprule
disturb $\lambda_\text{cov}$ & \makecell[c]{model} 
& \makecell[c]{avg acc.} 
& \makecell[r]{worst 10\% acc.} \\
\cmidrule(lr){1-4}
\multirow{3}{*}{0.0}
& Pure ERM     & \meanpm{0.07}{0.00}  & \meanpm{0.67}{0.02}  \\
& Classical OT & \meanpm{74.06}{1.20} & \meanpm{95.68}{0.15} \\
& PG-DRO       & \bestpm{78.10}{0.06} & \bestpm{98.93}{0.75} \\
\addlinespace
\multirow{3}{*}{1.0}
& Pure ERM     & \meanpm{10.03}{0.29} & \meanpm{29.07}{1.78} \\
& Classical OT & \meanpm{62.47}{0.11} & \meanpm{46.33}{0.63} \\
& PG-DRO       & \bestpm{65.31}{0.39} & \bestpm{48.07}{0.38} \\
\addlinespace
\multirow{3}{*}{2.0}
& Pure ERM     & \meanpm{17.14}{0.30} & \meanpm{30.20}{0.61} \\
& Classical OT & \meanpm{52.42}{0.26} & \meanpm{32.40}{0.38} \\
& PG-DRO       & \bestpm{54.47}{0.64} & \bestpm{32.73}{0.32} \\
\addlinespace
\multirow{3}{*}{3.0}
& Pure ERM     & \meanpm{20.01}{1.73} & \meanpm{25.67}{1.53} \\
& Classical OT & \meanpm{47.15}{0.27} & \meanpm{26.40}{3.27} \\
& PG-DRO       & \bestpm{51.39}{1.84} & \bestpm{32.53}{1.45} \\
\addlinespace
\multirow{3}{*}{4.0}
& Pure ERM     & \meanpm{21.03}{1.38} & \meanpm{22.93}{0.49} \\
& Classical OT & \meanpm{40.39}{0.16} & \meanpm{28.00}{0.21} \\
& PG-DRO       & \bestpm{49.73}{0.21} & \bestpm{39.93}{0.13} \\
\addlinespace
\multirow{3}{*}{5.0}
& Pure ERM     & \meanpm{21.68}{0.26} & \meanpm{23.27}{0.95} \\
& Classical OT & \meanpm{17.43}{0.13} & \meanpm{22.00}{0.21} \\
& PG-DRO       & \bestpm{32.53}{0.04} & \bestpm{26.60}{0.49} \\
\bottomrule
\end{tabular}}
\label{table_sim_cls}
\end{table}
\begin{figure*}[h]
    \centering
    \includegraphics[width=0.97\linewidth]{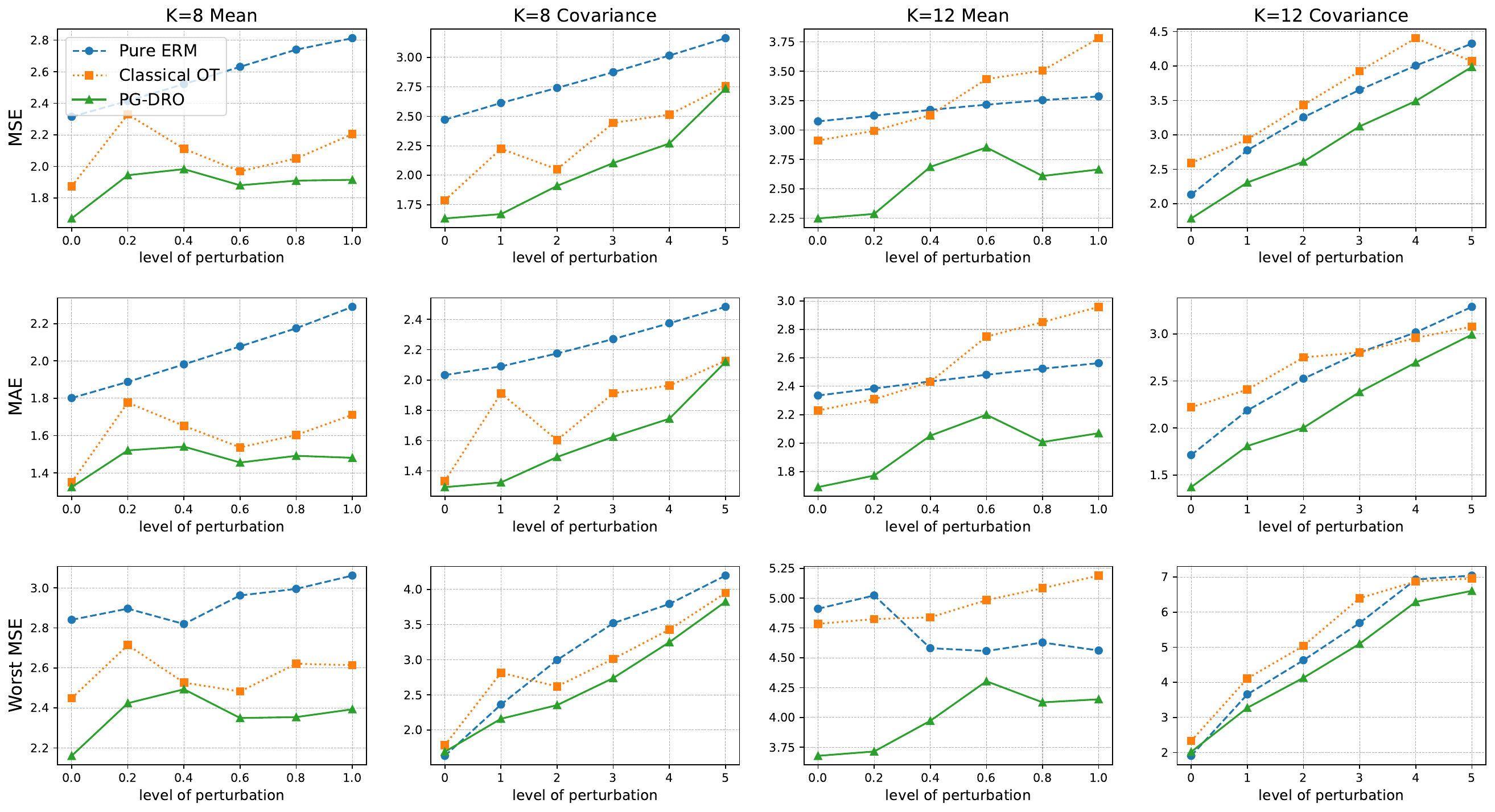}
    \caption{Comparison of robustness under prior shift: PG-DRO consistently improves both average accuracy and worst-case performance over ERM and OT baselines.}
    \label{fig_reg}
\end{figure*}
\subsubsection{Classification}
For classification, each training instance is a pair $(\mathbf{x},\mathbf{z})$, where $\mathbf{x}\in\mathbb{R}^d$ is the feature vector generated from the Gaussian mixture construction, and $\mathbf{z}\in\{1,\dots,C\}$ is its class label. We employ the standard cross-entropy loss \eqref{ce_loss}, with logits obtained from a linear classification head applied to the encoder output. For the ERM and classic OT baselines, logits are computed directly from the encoder–head network. In contrast, for our PG-DRO, logits are given by the robust dual $V_c(\mathbf{x};\theta)$, which integrates Sinkhorn-based class priors $\{\nu_c\}$ and solves $\lambda_c$ via Newton updates. To emulate the few-shot setting in the target domain, we randomly sample between $3$ and $8$ labeled support instances per class which are used both for supervision and for refining the class-specific priors in our method.
As shown in Table~\ref{table_sim_cls}, PG-DRO consistently improves both average accuracy and worst-case performance compared to ERM and classical OT. And the test is carried out on the target domain. This highlights the benefit of integrating adaptive OT priors with Sinkhorn DRO, yielding robustness not only in overall generalization but also in the most challenging shifted classes.

To evaluate whether adaptive OT effectively captures the correct relationship between the source and target domains, we present the results of correlation mapping in Figure \ref{fig_heatmap}. The heatmaps clearly show that as the number of support samples increases, the alignments progressively concentrate along the diagonal. This indicates that with sufficient support data, target classes can be more accurately anchored to their corresponding source classes. The trend becomes increasingly clear, demonstrating sharper and more reliable alignments between source and target domains with more support.
\begin{table*}[h]
\centering
\caption{Accuracy comparison for CIFAR-10 with Laplace Noise}
\resizebox{1\textwidth}{!}{
\begin{tabular}{@{}l *{3}{c} *{3}{c} *{3}{c}@{}}
\toprule
\multicolumn{1}{c}{} &
\multicolumn{3}{c}{Laplace1} &
\multicolumn{3}{c}{Laplace2} &
\multicolumn{3}{c}{Laplace5} \\
\cmidrule(lr){2-4}\cmidrule(lr){5-7}\cmidrule(lr){8-10}
Method & $k{=}1$ & $k{=}5$ & $k{=}10$ & $k{=}1$ & $k{=}5$ & $k{=}10$ & $k{=}1$ & $k{=}5$ & $k{=}10$ \\
\midrule
Few-shot
& \meanpm{10.85}{0.44} & \meanpm{27.25}{0.47} & \meanpm{59.16}{0.50}
& \meanpm{10.03}{0.13} & \meanpm{41.38}{0.41} & \meanpm{52.75}{0.48}
& \meanpm{10.05}{0.20} & \meanpm{25.58}{0.40} & \meanpm{32.24}{0.33} \\
SAA
& \meanpm{18.81}{1.48} & \meanpm{50.33}{0.61} & \meanpm{59.56}{0.40}
& \meanpm{21.91}{0.77} & \meanpm{48.03}{0.91} & \meanpm{58.93}{0.67}
& \meanpm{15.25}{0.42} & \meanpm{39.42}{0.83} & \meanpm{49.68}{0.55} \\
W-DRO
& \meanpm{20.23}{0.18} & \meanpm{49.34}{0.73} & \meanpm{60.41}{0.50}
& \meanpm{22.07}{0.95} & \meanpm{49.35}{0.73} & \meanpm{58.15}{0.66}
& \meanpm{15.96}{0.81} & \meanpm{40.49}{0.69} & \meanpm{50.21}{0.28} \\
PG-DRO
& \bestpm{20.88}{0.57} & \bestpm{50.45}{0.54} & \bestpm{61.54}{0.34}
& \bestpm{22.28}{0.71} & \bestpm{50.21}{0.86} & \bestpm{59.57}{0.45}
& \bestpm{17.71}{0.24} & \bestpm{41.53}{0.39} & \bestpm{50.58}{0.41} \\
\bottomrule
\end{tabular}
}
\label{tab_laplace}
\end{table*}
\begin{table*}[h]
\centering
\caption{Accuracy comparison for CIFAR-10 with Gaussian Noise}
\resizebox{1\textwidth}{!}{
\begin{tabular}{@{}l *{3}{c} *{3}{c} *{3}{c}@{}}
\toprule
\multicolumn{1}{c}{} &
\multicolumn{3}{c}{Gaussian1} &
\multicolumn{3}{c}{Gaussian2} &
\multicolumn{3}{c}{Gaussian5} \\
\cmidrule(lr){2-4}\cmidrule(lr){5-7}\cmidrule(lr){8-10}
Method & $k{=}1$ & $k{=}5$ & $k{=}10$ & $k{=}1$ & $k{=}5$ & $k{=}10$ & $k{=}1$ & $k{=}5$ & $k{=}10$ \\
\midrule
Few-shot
& \meanpm{10.07}{0.15} & \meanpm{26.11}{0.46} & \meanpm{59.38}{0.35}
& \meanpm{9.95}{0.22}  & \meanpm{41.95}{0.33} & \meanpm{52.87}{0.28}
& \meanpm{9.69}{0.34}  & \meanpm{25.48}{0.70} & \meanpm{33.25}{0.46} \\
SAA
& \meanpm{16.37}{0.91} & \meanpm{48.94}{0.67} & \meanpm{60.15}{0.42}
& \meanpm{17.80}{0.75} & \meanpm{48.68}{0.45} & \meanpm{58.37}{0.76}
& \meanpm{17.11}{0.75} & \meanpm{39.48}{0.47} & \meanpm{50.17}{0.70} \\
W-DRO
& \meanpm{19.11}{1.75} & \meanpm{50.70}{0.56} & \meanpm{59.12}{1.25}
& \meanpm{17.48}{1.02} & \meanpm{49.77}{0.53} & \meanpm{58.93}{0.70}
& \meanpm{17.45}{1.11} & \meanpm{39.45}{0.47} & \meanpm{49.73}{0.60} \\
PG-DRO
& \bestpm{20.53}{1.06} & \bestpm{53.62}{1.08} & \bestpm{60.78}{0.44}
& \bestpm{22.74}{1.44} & \bestpm{51.00}{1.15} & \bestpm{60.34}{0.48}
& \bestpm{18.80}{0.34} & \bestpm{41.86}{0.37} & \bestpm{50.47}{0.48} \\
\bottomrule
\end{tabular}
}
\label{tab_gauss}
\end{table*}
\subsubsection{Regression}
We follow the same setting to generate source and target samples $(\mathbf{x},\mathbf{z})$, where $\mathbf{x}$ is drawn from the synthetic geometry and the response is defined by $\mathbf{z}=\beta^\top \mathbf{x}+\varepsilon$. Thus we obtain paired data $(\mathbf{x},\mathbf{z})$ for regression.   Similarly, three regressors are trained with a shared encoder and a linear head. The ERM model and classical OT are optimized on training dataset using the \eqref{loss_erm} and \eqref{loss_ot} respectively. We also include our PG-DRO model minimizing a robust dual objective induced by the hierarchical OT prior. For adaptation, we sample a small number of labeled supports per target class as few-shot supervision and fine-tune accordingly. Evaluation is also carried out exclusively on the target domain. We report Mean Squared Error (MSE), Mean Absolute Error (MAE) set as well as worst-$10\%$ MSE, computed over the $10\%$ largest absolute errors in the target distribution.

We train all models with Adam Optimizer \cite{articleadam} using a short warm-up schedule. Figure \ref{fig_reg} reports the regression performance of PG-DRO across varying numbers of classes and noise types. The results show that PG-DRO consistently achieves competitive results in all settings. As expected, the overall regression error increases as the noise level becomes larger, but the relative advantage of PG-DRO remains clear, demonstrating its robustness to distributional perturbations.

\subsection{Image Classification}
\label{sec:exp_cls}
Beyond the synthetic experiments, we further evaluate the decision quality of PG-DRO on real-world vision benchmarks, including CIFAR-10 and CIFAR-100~\cite{cifar}, as well as mini-ImageNet and tiered-ImageNet \cite{ren2018meta}. Due to space constraints, we present the CIFAR results in the main text and defer the ImageNet-family results to Appendix~\ref{supp:exp}.

We use CIFAR-100 and as the source domain and CIFAR-10 as the target domain, and evaluate adaptation performance under extremely low-label conditions to answer \textbf{RQ1}. Specifically, we first pretrain a ResNet-18 encoder $E_\phi$ on CIFAR-100, and then transfer it to CIFAR-10. For each target class, we randomly sample $S \in \{1,5,10\}$ labeled supports in CIFAR-10 dataset. For each query $\textbf{x}$, we obtain its representation $m = E_\phi(\textbf{x})$, and then compute the logits $\sigma_c(\textbf{x}) = w_c^\top m + \epsilon_c$ according to different strategies.

To benchmark PG-DRO, we compare against two standard robust adaptation strategies. \begin{itemize}
  \item \textbf{Sample Average Approximation (SAA).} Augments each support $(\textbf{x}_i,\textbf{z}_i)$ with perturbations drawn from noise distribution $\mathcal D_{\text{noise}}$, and minimizes empirical risk $\mathbb E_{\eta\sim \mathcal D_{\text{noise}}}\big[\ell(\textbf{z}_i,f_\theta(\textbf{x}_i+\eta))\big]$
  via Monte Carlo approximation.
  \item \textbf{Wasserstein DRO (W-DRO).} Builds on the score $z_c(\textbf{x})$ but replaces the empirical distribution of supports with the worst-case distribution within a Wasserstein ball, thereby minimizing the worst-case expected loss.
\end{itemize}
\vspace{5pt}
For PG-DRO, we regard $\sigma_c(\textbf{x})$ as the objective $\ell$ in \eqref{dual_sinkhorn} and strictly follow the steps of PG-DRO. We first build class-adaptive priors by performing hierarchical OT between base classes in CIFAR-100 and few-shot supports in CIFAR-10. Sinkhorn couplings are aggregated into weights $w_{b,c}$, yielding Gaussian mixture priors $\nu_c = \sum_b w_{b,c}\,\mathcal N(\mu_b,\Sigma_b)$ for each target class. 

For a query $\mathbf{x}$, we kernelize its embedding $m=E_\phi(\mathbf{x})$ against $\nu_c$, obtaining a posterior mixture. Rather than using the plain logit $w_c^\top m+\epsilon_c$, we compute the robust logit $V_c(\mathbf{x};\theta)$ by evaluating the robust objective with respect to the posterior mixture, thereby shaping the decision rule under worst-case perturbations. Training then proceeds with standard cross-entropy \eqref{ce_loss} over $V_c(\mathbf{x})$. Unlike W-DRO, our priors vary across classes and are induced by the adaptive OT, enabling improved robust generalization under low-label conditions, which could exactly answer \textbf{RQ3}.

In addition, we also compare with the simplest baseline, where we directly fine-tune a pretrained ResNet-18 on the few-shot supports from CIFAR-10. To answer \textbf{RQ2}, evaluation is conducted on the standard CIFAR-10 test set on perturbed features where the encoder output is corrupted as $m^\dagger = m + \eta$, with different levels of Gaussian or Laplace noise for testing both generalization and robustness.

To conclude, all experiments are decision-focused and end-to-end. We train by minimizing cross-entropy on the robust logits $V_c(\mathbf x;\theta)$ with few-shot prototypes and therefore compare only to end-to-end baselines with the same training and evaluation interface under identical backbones and compute. Results in Table \ref{tab_laplace} and \ref{tab_gauss} show that PG-DRO consistently outperforms the baselines across different disturbance strengths, various shot settings, and under diverse noise conditions.

\section{Conclusion}
In this work, we introduced Prototype-Guided Distributionally Robust Optimization (PG-DRO), a framework that unifies hierarchical optimal transport priors with entropic DRO. By transferring structure from abundant base classes, PG-DRO constructs adaptive priors that embed robustness directly into the decision. This design enables the model not only to generalize from limited supervision but also to withstand distributional shifts and worst-case perturbations.

Theoretical analysis established that adaptive priors preserve consistency and reduce excess risk compared to the traditional DRO. Empirical results on synthetic and real datasets further validated our framework, showing consistent improvements in both accuracy and robustness under challenging few-shot settings.
Future work will extend PG-DRO to online regimes, updating class-adaptive priors and dual variables on the fly as new supports arrive, with targets of no-regret guarantees under nonstationary shifts.

\bibliography{aistats2026}
\bibliographystyle{plain}
\appendix
\onecolumn
\section{Pseudo-code for PG-DRO}

\begin{algorithm}[h]
\caption{PG-DRO: Prototype-Guided DRO}
\label{alg:pgdro}
\begin{algorithmic}[1]
\REQUIRE Base data $\mathcal{D}_\text{base}$, support set $\mathcal{S}$, model $f_\theta$, cost $c$, radius $\rho$, temperature $\varepsilon$
\STATE \textbf{Phase I: Prototype-Guided Priors}
\STATE Compute statistics $(\mu_b,\Sigma_b)$ for each base class $b$
\STATE Solve hierarchical OT between $\mathcal{S}$ and $\mathcal{D}_\text{base}$ $\rightarrow$ transport plan $T$
\STATE For each novel class $c$: set prior $\nu_c = \sum_b \tilde w_{bc}\mathcal{N}(\mu_b,\Sigma_b)$

\STATE \textbf{Phase II: Robust Training}
\FOR{each minibatch $\{(x,c^\star)\}$}
  \FOR{each sample $x$ and class $c$}
    \STATE Define Gibbs kernel $Q^{\nu_c}_{\varepsilon,x}$ from cost $c$ and prior $\nu_c$
    \STATE Solve 1D convex problem $\min_{\lambda\ge0}\phi_c(\lambda)$
    \STATE Robust logit $V_c(x)=\phi_c(\lambda_c^\star)$
  \ENDFOR
  \STATE Loss $L(x,c^\star)=-\log \frac{e^{V_{c^\star}}}{\sum_c e^{V_c}}$
  \STATE Update $\theta \leftarrow \theta - \eta\nabla_\theta L$
\ENDFOR
\end{algorithmic}
\end{algorithm}

\section{Proof of Theorem 4.2}
\label{supp:proof4.2}

Based on the following assumption, we proved Algorithm 4.2.
\begin{assumption}[Smooth lower level] \label{ass:1}
$C(w)$ is $C^1$ and Lipschitz in a neighborhood of $w^\star$.
\end{assumption}
\begin{assumption}[Interior/positivity] \label{ass:2}
$\tau>0$ and the optimal Sinkhorn coupling at $w^\star$ has strictly positive entries bounded away from zero; also $\min_b w_b^\star\ge \gamma>0$.
\end{assumption}
\begin{assumption}[Few-shot sampling]\label{ass:3}
$\mathbb{E}\|\hat b-b\|_1=\mathcal O(N_t^{-1/2})$ for a support size $N_t$.
\end{assumption}
\begin{assumption}[Regularity for stability of $V$]\label{ass:stab}
Fix $\varepsilon>0$ and define $k_x(y):=\exp(-c(x,y)/\varepsilon)$. There exist constants
$0<m\le M<\infty$, $L_k<\infty$, $D<\infty$, $B_f<\infty$, $L_f<\infty$ such that for every $x$:
\begin{enumerate}
\item $m\le k_x(y)\le M$ and $\mathrm{Lip}_y(k_x)\le L_k$ on $Y$;
\item either $Y$ is compact with $\mathrm{diam}(Y)\le D$, or all measures considered have uniformly bounded first moment so that 1\textnormal{-}Lipschitz test functions are uniformly bounded by $D$;
\item $f_\theta(x,\cdot)$ is bounded and Lipschitz on $Y$: $|f_\theta(x,y)|\le B_f$ and $\mathrm{Lip}_y(f_\theta(x,\cdot))\le L_f$.
\end{enumerate}
\end{assumption}

\textbf{Theorem 4.2} (Core contraction of DRO objective for $V$ under adaptive OT)  
\textit{Fix a target class $c$ and omit the subscript $c$ for brevity. 
Let $w^{(t+1)}=(1-\eta_t)w^{(t)}+\eta_t\widehat T(w^{(t)})$ with $\eta_t\in(0,1]$, 
and define $\nu^{(t)}=\sum_b w^{(t)}_{b}\mu_b$ and $\Delta_t(\theta):=|V_{\nu^{(t)}}(\theta)-V_{\nu^\star}(\theta)|$. 
Suppose the OT map $T$ is locally contractive near $w^\star$ with contraction rate $\kappa\in(0,1)$. 
Then, for all $t$ in the contraction neighborhood,
\begin{equation}
    \Delta_{t+1}(\theta)\ \le\ (1-\eta_t\kappa)\Delta_t(\theta)\ +\ \mathcal O(N^{-1/2}),
\end{equation}
where $\kappa$ is the contraction constant of the population map $T$ around $w^\star$, and $N$ is the number of few-shot samples.}

\begin{proof}

We first prove the following lemma:

\begin{lemma}\label{lem:stability}
Under Assumption~\ref{ass:stab}, there exists a constant $C<\infty$ such that for all mixtures $\mu,\nu$ and all parameters $\theta$,
\begin{equation}\label{eq:stability}
\big|V_\mu(\theta)-V_\nu(\theta)\big|\ \le\ CW_1(\mu,\nu).
\end{equation}
\end{lemma}

\begin{proof}
Write the tilted normalization $\Xi^x(\nu)(dy):=\dfrac{k_x(y)\nu(dy)}{\int k_xd\nu}$.

For any $1$-Lipschitz $h:Y\to\mathbb R$,
\begin{equation*}
\int hd(\Xi^x(\nu)-\Xi^x(\mu))
=\frac{\langle hk_x,\nu\rangle}{\langle k_x,\nu\rangle}-\frac{\langle hk_x,\mu\rangle}{\langle k_x,\mu\rangle}
=:A+B.
\end{equation*}
Using $m\le \langle k_x,\cdot\rangle\le M$, $\|h\|_\infty\le D$, and
$\mathrm{Lip}(hk_x)\le M\cdot 1+L_k\cdot D$, we obtain
\begin{equation*}
|A|\le \tfrac{M+L_kD}{m}W_1(\nu,\mu),\qquad
|B|\le \tfrac{DML_k}{m^2}W_1(\nu,\mu).
\end{equation*}
Taking the supremum over $h$ and invoking Kantorovich–Rubinstein duality yields
\begin{equation}\label{eq:XiLip}
W_1\big(\Xi^x(\nu),\Xi^x(\mu)\big)\ \le\ K_xW_1(\nu,\mu),\quad K_x:=\frac{M+L_k D}{m}+\frac{DML_k}{m^2}.
\end{equation}

Fix any $\lambda_0>0$ and set $g(y):=f_\theta(x,y)/(\lambda_0\varepsilon)$.
Then $\mathrm{Lip}(g)=L_f/(\lambda_0\varepsilon)$ and $|g|\le B_f/(\lambda_0\varepsilon)$, so
$\mathrm{Lip}(e^{g})\le e^{B_f/(\lambda_0\varepsilon)}\cdot L_f/(\lambda_0\varepsilon)$ and
$\min\{\mathbb E[e^{g}]\}\ge e^{-B_f/(\lambda_0\varepsilon)}$.
Hence
\begin{equation*}
\Big|\lambda_0\varepsilon\log\mathbb E_{\Xi^x(\nu)}[e^{g}] -\lambda_0\varepsilon\log\mathbb E_{\Xi^x(\mu)}[e^{g}]\Big| \le L_fe^{\frac{2B_f}{\lambda_0\varepsilon}}\ W_1\big(\Xi^x(\nu),\Xi^x(\mu)\big).
\end{equation*}
Combine with \eqref{eq:XiLip} to get
\begin{equation}\label{eq:PhiLip}
\Big|\lambda_0\varepsilon\log\mathbb E_{\Xi^x(\nu)}[e^{g}] -\lambda_0\varepsilon\log\mathbb E_{\Xi^x(\mu)}[e^{g}]\Big| \le L_fe^{\frac{2B_f}{\lambda_0\varepsilon}}\ K_x\ W_1(\nu,\mu).
\end{equation}

By definition of $V$ (the inner minimization over $\lambda>0$), for any fixed $\lambda_0>0$,
\begin{equation*}
V_\nu(x)\ \le\ \lambda_0\rho+\lambda_0\varepsilon\log\mathbb E_{\Xi^x(\nu)}[e^{g}],\qquad
V_\mu(x)\ \le\ \lambda_0\rho+\lambda_0\varepsilon\log\mathbb E_{\Xi^x(\mu)}[e^{g}],
\end{equation*}
so
\begin{equation*}
|V_\nu(x)-V_\mu(x)|\ \le\ \Big|\lambda_0\varepsilon\log\mathbb E_{\Xi^x(\nu)}[e^{g}]-\lambda_0\varepsilon\log\mathbb E_{\Xi^x(\mu)}[e^{g}]\Big|.
\end{equation*}
Apply \eqref{eq:PhiLip} and take $\sup_x K_x$ (finite by Assumption~\ref{ass:stab}) to conclude
\begin{equation*}
|V_\nu(\theta)-V_\mu(\theta)|\ \le\ CW_1(\nu,\mu),
\end{equation*}
with $C$ as stated.
\end{proof}

Then we prove the theorem:

Fix a dictionary of measures $\{\mu_b\}_{b=1}^B$. Any prior is a mixture
$\nu = \sum_{b=1}^B w_b\mu_b$ with $w\in\Delta_B := \{w\ge 0,\ \sum_b w_b=1\}$.
The upper level OT uses entropic OT with cost matrix $C(w)\in\R^{B\times m}$ and kernel
\begin{equation*}
K(w) := \exp\big(-C(w)/\tau\big)\in \R_{++}^{B\times m},\qquad \tau>0.
\end{equation*}
Its Sinkhorn coupling has the form
\begin{equation*}
\Pi(w)=\mathrm{diag}(u)K(w)\mathrm{diag}(v),\quad
u\odot(K(w)v)=w,\quad v\odot(K(w)^\top u)=b,
\end{equation*}
where $b\in\Delta_m$ is the population target marginal. The population map from $w$ to new mixture weights is
\begin{equation*}
T(w)=\Norm\big(R\Pi(w)\one_m\big)\in\Delta_B,
\end{equation*}
with a fixed linear aggregator $R$ (often $R=I$). The empirical map $\widehat T(w)$ replaces $b$ by an empirical $\hat b$.
The algorithm uses the damped update
\begin{equation}\label{eq:update}
w^{(t+1)}=(1-\eta_t)w^{(t)}+\eta_t\widehat T(w^{(t)}),\qquad \eta_t\in(0,1].
\end{equation}
Distances between mixtures are measured by $W_1$.

Write the upper-level Lagrangian
\begin{equation*}
\mathcal L(\Pi,\alpha,\beta;w)=\langle C(w),\Pi\rangle+\tau\sum_{ij}\Pi_{ij}(\log\Pi_{ij}-1)+\alpha^\top(\Pi\one-w)+\beta^\top(\Pi^\top\one-b).
\end{equation*}

Because $\tau>0$, the Hessian in $\Pi$ on the constraint tangent space obeys $\nabla^2_{\Pi\Pi}\succeq \tau\mathrm{Diag}(1/\Pi^\star)\succeq \tau/m_\Pi\cdot I$ with $m_\Pi=\min_{ij}\Pi^\star_{ij}>0$ by Assumption~\ref{ass:2}.
The KKT system $F(\Pi,\alpha,\beta;w)=0$ has a locally unique solution $(\Pi^\star(w),\alpha^\star(w),\beta^\star(w))$ by the implicit function theorem.
Differentiating in $w$ yields
\begin{equation*}
D\Pi^\star(w^\star)= -\big(\nabla^2_{\Pi\Pi}\mathcal L\big)^{-1}\nabla^2_{\Pi w}\mathcal L,
\end{equation*}
so $\|D\Pi^\star(w^\star)\|\le L_C/(\tau/m_\Pi)=:L_\Pi$ by Assumption~\ref{ass:1}.
The map $w\mapsto T(w)=\Norm(R\Pi^\star(w)\one)$ is a composition of linear maps and a smooth normalization in an interior neighborhood, hence
$\|DT(w^\star)\|\le c_{\mathrm{agg}}L_\Pi$ for some finite constant $c_{\mathrm{agg}}$.
Choosing (or verifying) parameters so that $c_{\mathrm{agg}}L_\Pi<1$ and shrinking the neighborhood if necessary gives
$\sup_{w\in\mathcal U}\|DT(w)\|\le 1-\kappa$, which implies 
\begin{equation}\label{eq:pop_contraction}
\|T(w)-T(w^\star)\|_1 \ \le\ (1-\kappa)\|w-w^\star\|_1.
\end{equation}
by the mean-value theorem.

Applying the implicit function argument w.r.t.\ the marginal $b$ gives
$\|\widehat\Pi(w)-\Pi(w)\|\le L_b\|\hat b-b\|_1$.
Aggregation and normalization are Lipschitz in an interior neighborhood, hence
$\|\widehat T(w)-T(w)\|_1\le c_{\mathrm{agg}}L_b\|\hat b-b\|_1$.
Taking expectations and using Assumption~\ref{ass:3} yields \begin{equation}\label{eq:emp_deviation}
\mathbb E\big[\|\widehat T(w)-T(w)\|_1\big|N_t\big]\ \le\ c_T\mathbb E\|\hat b-b\|_1\ =\ \mathcal O(N_t^{-1/2}).
\end{equation}

With a fixed dictionary supported on a set of diameter $D$, we have
$W_1(\sum_b (w_b-w_b')\mu_b)\le (D/2)\|w-w'\|_1$. Set $L_{\mathrm{mix}}:=D/2$ to obtain 
\begin{equation}\label{eq:mix_equiv}
W_1\Big(\sum_b w_b\mu_b,\sum_b w_b'\mu_b\Big)\ \le\ L_{\mathrm{mix}}\|w-w'\|_1,
\end{equation}
where all $w,w'\in\Delta_B$.

By convexity of $W_1$ and the update \eqref{eq:update},
\begin{equation*}
W_1(\nu^{(t+1)},\nu^\star)\le (1-\eta_t)W_1(\nu^{(t)},\nu^\star)+\eta_tW_1(\widehat T(w^{(t)}),\nu^\star).
\end{equation*}
Using the triangle inequality with Steps~1--3 gives
$W_1(\widehat T(w^{(t)}),\nu^\star)\le (1-\kappa)W_1(\nu^{(t)},\nu^\star)+B+\xi_t$,
which yields 
\begin{equation}\label{eq:recursion}
E_{t+1}\ \le\ (1-\eta_t\kappa)E_t+\eta_t\big(B+\xi_t\big),
\qquad \mathbb E[\xi_t|N_t]=\mathcal O(N_t^{-1/2}).
\end{equation}
In particular, if $B=0$ and $N_t\to\infty$ (or $\sum_t \eta_t\mathbb E\xi_t<\infty$), then $E_t\to 0$; with constant $\eta_t\equiv \eta$,
\begin{equation*}
\limsup_{t\to\infty} E_t\ \le\ \frac{B+\overline\xi}{\kappa},\qquad \overline\xi:=\limsup_t \mathbb E\xi_t.
\end{equation*}

Apply \eqref{eq:stability} to transfer bounds from $E_t$ to $\Delta_t(\theta)$; standard epi-convergence arguments then yield convergence of $L_t$ and minimizers.

\begin{equation*}
\Delta_t(\theta):=|V_{\nu^{(t)}}(\theta)-V_{\nu^\star}(\theta)|\ \le\ CE_t,
\end{equation*}
so $\Delta_t(\theta)$ satisfies the same contraction conclusions as $E_t$. Under standard integrability/strict convexity already assumed in the DRO dual, $L_t(\theta)\to L^\star(\theta)$ and (when unique) minimizers $\theta_t\to\theta^\star$.
\end{proof}

\paragraph{Multi-class extension.}
Apply the theorem per class $c$ and sum over $c$; all constants can be chosen uniformly on a neighborhood where each class remains interior.

\section{Proof of Theorem 4.3}

\textbf{Theorem 4.3} (Consistency)
\textit{Assume the ground cost $c(x,y)$ is continuous and bounded below, $f_\theta(\mathbf{x},\mathbf{y})$ is continuous in $\mathbf{x}$, $\mathbf{y}$ and locally Lipschitz in $\theta$, the output space $\mathcal Y$ is compact, then for each $\mathbf{x}$ and class $c$,
\begin{equation}
V_c^{(N)}(\mathbf{x};\theta)\to V_c^\star(\mathbf{x};\theta)~~\text{and}~~
\lambda_c^{(N)}(\mathbf{x})\to\lambda_c^\star(\mathbf{x}),
\end{equation}
where $V_c^{(N)}$ (resp.\ $V_c^\star$) is the Sinkhorn DRO dual with kernel induced by $\nu_c^{(N)}$ (resp.\ $\nu_c^\star$), and $\lambda_c^{(N)}$ (resp.\ $\lambda_c^\star$) is any minimizer of the corresponding dual variable.}

\begin{proof}
By the definition of Gibbs kernel
\begin{equation*}
dQ^\nu_{\varepsilon,x}(y)=
\frac{e^{-c(x,y)/\varepsilon}}
{\int e^{-c(x,u)/\varepsilon}d\nu(u)}d\nu(y).
\end{equation*}
If $\nu^{(N)}\Rightarrow\nu^\star$, then $Q^{\nu^{(N)}}_{\varepsilon,x}\Rightarrow Q^{\nu^\star}_{\varepsilon,x}$. 
Indeed, since $c$ is continuous and bounded below, 
$k_x(y)=e^{-c(x,y)/\varepsilon}$ is bounded and continuous, hence
$\int k_xd\nu^{(N)}\to \int k_xd\nu^\star$. 
For any bounded continuous $g$,
\begin{equation*}
\int gdQ^{\nu^{(N)}}_{\varepsilon,x}
=\frac{\int gk_xd\nu^{(N)}}{\int k_xd\nu^{(N)}}
\to \frac{\int gk_xd\nu^\star}{\int k_xd\nu^\star}
=\int gdQ^{\nu^\star}_{\varepsilon,x}.
\end{equation*}

For simplicity, we denote 
\begin{equation*}
\Phi^{(N)}_c(\lambda;x)
=\lambda\rho+\lambda\varepsilon\log
\mathbb{E}_{y\sim Q^{\nu^{(N)}_c}_{\varepsilon,x}}
\exp\{f(x,y)/(\lambda\varepsilon)\},
\end{equation*}
and similarly $\Phi^\star_c$ with $\nu_c^\star$. Since for any compact $\Lambda\subset(0,\infty)$,
\begin{equation*}
\sup_{\lambda\in\Lambda}
|\Phi^{(N)}_c(\lambda;x)-\Phi^\star_c(\lambda;x)|\to 0.
\end{equation*} $f$ bounded and continuous $\Rightarrow g_\lambda(y)=\exp(f(x,y)/(\lambda\varepsilon))$ bounded continuous and equicontinuous in $\lambda$, so expectations converge uniformly. 

From \cite[Thm.~1(IV), Lem.~EC.5]{wang2025sinkhorn}, we have
\begin{equation*}
\lim_{\lambda\to\infty}\Phi_\nu(\lambda;x)
=\lim_{\lambda\to\infty}\big(\lambda\rho+\lambda\varepsilon \mathbb{E}_x \log \mathbb{E}_{Q_{x,\varepsilon}} e^{f/(\lambda\varepsilon)}\big)
=+\infty\quad (\rho>0),
\end{equation*}
so minimizers cannot escape to infinity. Similarly, the condition for $\lambda^\star=0$ is characterized therein. Thus all minimizers lie in some compact interval $\Lambda\subset(0,\infty)$ independent of $N$. 
Moreover, by lemma 3.1, $\Phi_\nu(\cdot;x)$ is convex in $\lambda$, strictly convex unless $f$ is a.s.\ constant, hence the minimizer is unique (except trivial degeneracy).
Therefore $\Phi^{(N)}_c\to\Phi^\star_c$ uniformly on $\Lambda$. 
By standard convex optimization stability (epi-convergence plus uniqueness), 
\begin{equation*}
\lambda_c^{(N)}(x)\to\lambda_c^\star(x).
\end{equation*}

Since we already proved uniform convergence and argmin convergence, we have
\begin{equation*}
V_c^{(N)}(x)=\Phi^{(N)}_c(\lambda_c^{(N)}(x);x)\to 
\Phi^\star_c(\lambda_c^\star(x);x)=V_c^\star(x).
\end{equation*}

Besides, in the multi-class setting we use the robust logits $V_c(x)$ inside the softmax cross-entropy. Under $f_\theta(x,\cdot)$ is continuous (hence bounded), the terms $\exp(V_c^{(N)})$ are uniformly integrable.  Hence by dominated convergence,
\begin{equation*}
\log\sum_c e^{V_c^{(N)}(x;\theta)}\to 
\log\sum_c e^{V_c^\star(x;\theta)}.
\end{equation*}
Thus $\mathcal L_N(\theta)\to \mathcal L_\star(\theta)$ pointwise and locally uniformly in $\theta$. 
By epi-convergence of convex functions, any limit point of minimizers $\theta_N$ lies in $\arg\min \mathcal L_\star$, and uniqueness implies $\theta_N\to\theta^\star$.
\end{proof}

\section{Hyperparameters}

To produce our experimental results, the hyperparameter are list as follows:

\textbf{Classification: } 
We use $K=8$ classes with input dimensions $d=10$. We sample $n_{\text{train}}=6000$ and $n_{\text{test}}=3000$ points, with the target domain mixture ratio set to $\alpha_{\text{test}}=0.15$. For distributional perturbations, we apply a mean shift of 0.6, covariance scaling of 1.0, and a 15-degree rotation. Optimal transport is regularized with entropic weight 0.2. For DRO, we set the $\varepsilon_{\text{sample}}=1.0$, $\varepsilon_{\text{class}}=0.8$, and balance parameter $\rho$ as 1.0. The model is trained using a hidden dimension of 256, batch size 256, learning rate $10^{-3}$, and 200 epochs, with Newton iterations set to 8. We also report performance on the worst 10\% quantile of test samples to assess robustness.

\textbf{Regression: } 
Similarly as the classification task, we also use $K=8$ classes with input dimension $d=10$. We sample $n_{\text{train}}=6000$ and $n_{\text{test}}=3000$ with target mixture $\alpha_{\text{test}}=0.20$. Shifts are applied with mean shift $0.6$, covariance scaling $1.15$, rotation $15^\circ$. Optimal transport uses entropic weight $0.2$; hierarchical OT priors use $\varepsilon_{\text{sample}}=1.0$, $\varepsilon_{\text{class}}=0.8$ (200 iterations), and covariance inflation $3.0$. The regression targets use $w\sim\mathcal{N}(0,I_{d})$ and additive noise $\varepsilon\sim\mathcal{N}(0,0.5^2)$. Training uses a single hidden layer of size $256$, batch size $256$, learning rate $10^{-3}$, Huber loss with beta $=1.0$, and $200$ epochs. The DRO prior penalty employs temperature $0.1$ and loss weight $\lambda=1.0$. Evaluation reports RMSE, MAE and worst-$10\%$ RMSE.

\textbf{Image Classification:} We use a ResNet-18 encoder with feature dimension $512$ and a linear head. 
Few-shot supports are sampled with $S \in \{1,5,10\}$ per class. Training uses batch size $64$, learning rate $10^{-4}$, and $50$ epochs. In PG-DRO, we set the sample-level softmin temperature $\varepsilon_{\text{sample}}=0.1$, the class-level entropic weight $\varepsilon_{\text{class}}=0.1$ with $100$ Sinkhorn iterations, and a covariance inflation factor of $1.5$. The ambiguity radius is fixed to $\rho=5.0$, and the robust logit solver runs $5$ Newton iterations. 
Evaluation is performed on both clean features and perturbed ones, with Gaussian or Laplace noise of relative radius $\epsilon\in \{1,2,5\}$. 

\section{Additional experimental results}

In the main text we reported results on CIFAR-based transfer tasks. This appendix provides additional few-shot cross-domain image classification results on ImageNet-family benchmarks under feature-space distribution shifts. Following Section \ref{sec:exp_cls}, we evaluate robustness by corrupting the encoder representation with additive Laplace or Gaussian noise. We report mean accuracy (in \%) $\pm$ standard deviation over repeated runs for $k\in\{1,5,10\}$ labeled supports per class.

Overall, PG-DRO consistently achieves the best robust accuracy across all three transfer settings and both noise types. The advantage over robust baselines typically grows as more supports are available, which aligns with the motivation that adaptive OT priors can better leverage limited support information to form class-specific reference distributions for Sinkhorn DRO. Besides, we can observe when using more complex and stronger datasets, the performance gains of PG-DRO become more pronounced over merely simple CIFAR-10 based baselines, which further demonstate the advantage of PG-DRO.

\begin{table*}[h]
\centering
\caption{Accuracy comparison for miniImageNet$\rightarrow$CIFAR-100 with noise}

\begin{tabular}{@{}l *{3}{c} *{3}{c}@{}}
\toprule
\multicolumn{1}{c}{} &
\multicolumn{3}{c}{Laplace2} &
\multicolumn{3}{c}{Gaussian2} \\
\cmidrule(lr){2-4}\cmidrule(lr){5-7}
Method & $k{=}1$ & $k{=}5$ & $k{=}10$ & $k{=}1$ & $k{=}5$ & $k{=}10$ \\
\midrule
Few-shot
& \meanpm{1.05}{0.13} & \meanpm{7.95}{0.27} & \meanpm{11.55}{0.30}
& \meanpm{1.12}{0.10} & \meanpm{8.02}{0.21} & \meanpm{12.10}{0.26} \\
SAA
& \meanpm{4.95}{0.20} & \meanpm{12.14}{0.57} & \meanpm{17.32}{0.47}
& \meanpm{4.84}{0.14} & \meanpm{12.49}{0.57} & \meanpm{18.13}{0.49} \\
W-DRO
& \meanpm{6.93}{0.31} & \meanpm{14.53}{0.71} & \meanpm{22.97}{0.71}
& \meanpm{6.10}{0.21} & \meanpm{13.61}{1.17} & \meanpm{21.37}{1.08} \\
PG-DRO
& \bestpm{7.45}{0.36} & \bestpm{18.78}{0.64} & \bestpm{28.11}{0.62}
& \bestpm{7.84}{0.42} & \bestpm{17.86}{0.52} & \bestpm{26.90}{0.48} \\
\bottomrule
\end{tabular}

\end{table*}

\begin{table*}[h]
\centering
\caption{Accuracy comparison for tieredImageNet$\rightarrow$CIFAR-100 with Noise}
\begin{tabular}{@{}l *{3}{c} *{3}{c}@{}}
\toprule
\multicolumn{1}{c}{} &
\multicolumn{3}{c}{Laplace2} &
\multicolumn{3}{c}{Gaussian2} \\
\cmidrule(lr){2-4}\cmidrule(lr){5-7}
Method & $k{=}1$ & $k{=}5$ & $k{=}10$ & $k{=}1$ & $k{=}5$ & $k{=}10$ \\
\midrule
Few-shot
& \meanpm{1.01}{0.18} & \meanpm{8.41}{0.25} & \meanpm{11.61}{0.35}
& \meanpm{1.15}{0.14} & \meanpm{7.93}{0.29} & \meanpm{11.96}{0.21} \\
SAA
& \meanpm{6.20}{0.16} & \meanpm{23.94}{0.35} & \meanpm{31.15}{0.25}
& \meanpm{5.40}{0.45} & \meanpm{24.65}{0.37} & \meanpm{32.58}{0.32} \\
W-DRO
& \meanpm{7.06}{0.29} & \meanpm{25.66}{0.58} & \meanpm{34.85}{0.30}
& \meanpm{6.92}{0.33} & \meanpm{24.31}{0.58} & \meanpm{34.35}{0.51} \\
PG-DRO
& \bestpm{9.21}{0.47} & \bestpm{29.95}{0.44} & \bestpm{40.53}{0.36}
& \bestpm{9.26}{0.24} & \bestpm{30.25}{0.41} & \bestpm{40.05}{0.04} \\
\bottomrule
\end{tabular}

\end{table*}

\begin{table*}[h]
\centering
\caption{Accuracy comparison for tieredImageNet$\rightarrow$miniImageNet with Noise}
\begin{tabular}{@{}l *{3}{c} *{3}{c}@{}}
\toprule
\multicolumn{1}{c}{} &
\multicolumn{3}{c}{Laplace2} &
\multicolumn{3}{c}{Gaussian2} \\
\cmidrule(lr){2-4}\cmidrule(lr){5-7}
Method & $k{=}1$ & $k{=}5$ & $k{=}10$ & $k{=}1$ & $k{=}5$ & $k{=}10$ \\
\midrule
Few-shot
& \meanpm{1.12}{0.09} & \meanpm{9.43}{0.14} & \meanpm{13.99}{0.34}
& \meanpm{1.03}{0.08} & \meanpm{9.10}{0.13} & \meanpm{14.03}{0.36} \\
SAA
& \meanpm{20.57}{0.41} & \meanpm{30.89}{0.76} & \meanpm{39.60}{1.28}
& \meanpm{15.37}{0.27} & \meanpm{27.47}{1.55} & \meanpm{38.15}{0.32} \\
W-DRO
& \meanpm{19.79}{0.43} & \meanpm{33.88}{0.37} & \meanpm{41.87}{0.59}
& \meanpm{16.48}{0.93} & \meanpm{32.30}{0.35} & \meanpm{43.77}{0.12} \\
PG-DRO
& \bestpm{22.10}{0.75} & \bestpm{40.04}{0.39} & \bestpm{49.66}{0.70}
& \bestpm{19.95}{0.44} & \bestpm{39.73}{1.15} & \bestpm{50.84}{0.85} \\
\bottomrule
\end{tabular}
\end{table*}

\label{supp:exp}
\end{document}